\documentclass[twoside,12pt]{article}

\usepackage{jmlr2e}
\usepackage{animate}
\usepackage{amsmath}
\usepackage{natbib}
\usepackage{physics}
\usepackage{bm}
\usepackage{cleveref}
\usepackage{mathrsfs}
\usepackage{multirow}
\usepackage{float}
\usepackage{mathtools,amssymb,amsmath}
\usepackage{amsfonts}
\usepackage[dvipsnames]{xcolor}
\RequirePackage{bm}
\usepackage[linesnumbered,ruled]{algorithm2e}
\usepackage{algpseudocode}
\usepackage{multirow, makecell}
\newcommand{\Ub}{\mathbf{U}}
\newcommand{\Vb}{\mathbf{V}}

\newcommand{\mb}{\mathbf}

\firstpageno{1}

\begin{document}
	
	\title{Adaptive Weighted Multi-View Clustering}
	
	\author{\name Shuo Shuo Liu \email shuoshuo.liu@psu.edu \\
 \addr Department of Statistics, The Pennsylvania State University, USA\\
		\name Lin Lin \email lynn.lin@duke.edu \\
		\addr Department of Biostatistics and Bioinformatics, Duke University, USA\\}
	
	
	\maketitle
	
	\begin{abstract}%
		Learning multi-view data is an emerging problem in machine learning research, and nonnegative matrix factorization (NMF) is a popular dimensionality-reduction method for integrating information from multiple views. These views often provide not only consensus but also complementary information.
		However, most multi-view NMF algorithms assign equal weight to each view or tune the weight via line search empirically, which can be infeasible without any prior knowledge of the views or computationally expensive. 
		In this paper, we propose a weighted multi-view NMF (WM-NMF) algorithm.  In particular, we aim to address the critical technical gap, which is to learn both view-specific weight and observation-specific reconstruction weight to quantify each view’s information content. The introduced weighting scheme can alleviate unnecessary views' adverse effects and enlarge the positive effects of the important views by assigning smaller and larger weights, respectively. 
		Experimental results confirm the effectiveness and advantages of the proposed algorithm in terms of achieving better clustering performance and dealing with the noisy data compared to the existing algorithms.
	\end{abstract}
	
	
\section{Introduction}
\label{sec:intro}

Learning multi-view data is an emerging problem in machine learning research, as multi-view data become more and more common in many real-world applications. 
For example, the multi-omics data are now ubiquitous where different biological layers such as genomics, epigenomics, transcriptomics, and proteomics can be obtained from the same set of objects~\citep{Hasin2017, 10.1371/journal.pgen.1009398}. In those scenarios, the same set of objects has different views collected from different measuring methods or modalities, where any particular single-view data may be inadequate to comprehensively describe the information of all the objects. 
Hence, one major goal of multi-view unsupervised learning is to search for a consensus clustering across views so that similar objects are grouped into the same cluster and dissimilar objects are separated into
different clusters.
In the literature, such a learning problem is called  multi-view clustering \citep{bickel2004multi}. 

There are mainly two groups of approaches in the existing literature: generative (model-based) and discriminative (similarity-based and dimension reduction-based) \citep{Rappoport2018}. For the generative approach, we typically use the mixture model and regression-based matrix factorization. The idea is to model each data view's probabilistic distribution and obtain a common clustering result by either allowing all views to share the same priors or derived from a shared latent factors~\citep{lashkari2008convex, shen2009integrative, tzortzis2009convex, tzortzis2010multiple, savage2010discovering, lock2013bayesian, gabasova2017clusternomics}. An advantage of the generative approach is that it provides a nice interpretation of what the cluster is built on, but this approach is more computationally expensive in the context of multi-view learning. The discriminative approach focuses on the objective function that optimizes the average similarities within clusters and dissimilarities between clusters.
Different objective functions result in different methods, such as multi-view spectral clustering ~\citep{wang2013multi, kumar2011co,kumar2011coN}, nonnegative matrix factorization for multi-view clustering~\citep{Liu2013, 6909425, 10.1093/bioinformatics/btv544, huang2014robust, 10.1093/nar/gks725}, and canonical correlation analysis~\citep{10.1145/1553374.1553391, klami2013bayesian, lai2000kernel, witten2009extensions, chen2013structure}.
The discriminative approach generally involves non-convex objective functions and it might be hard to find good solutions.

Nonnegative matrix factorization (NMF) is a well-known algorithm for dimension reduction and feature extraction for nonnegative data. Unlike other matrix factorization techniques~\citep{Golub1970, abdi2010principal, Zhao2015}, NMF provides a more intuitive and interpretable understanding through the \textit {parts-based representation}: a data point can be represented by only a few activated basis elements~\citep{turk1991eigenfaces, lee1999learning}. NMF has been shown the advantages of extracting sparse and meaningful information from high-dimensional data~\citep{lee1999learning}. The theoretical analysis further reveals the equivalence of NMF and spectral clustering and K-means clustering~\citep{ding2005equivalence}. Thus, NMF can also be viewed as a clustering method.  The multi-view NMF (MultiNMF)~\citep{Liu2013} is an extension of NMF problem to integrate multiple nonnegative data matrices obtained from a common set of data points.
The framework of MultiNMF attempts to approximate each view  with some constraints in order to obtain both consensus and view-specific information. Existing related methods tackle this problem with different objective functions motivated by different applications~\citep{10.1093/nar/gks725, 10.1093/bioinformatics/bts476, 10.1371/journal.pcbi.1004042, 10.1093/bioinformatics/btv544}. However, most existing MultiNMF related  methods either assume that all views are equally important or the view-specific weights are known a priori in deriving the consensus clustering. In practice, such an assumption may not be valid as we often have noisy datasets.

The aim of this paper is to design an effective multi-view NMF algorithm that not only can perform multi-view clustering but also quantify each view’s weight and each observation's reconstruction weight by learning the corresponding relative values across all views.  We expect this weighting mechanism to improve the clustering performance over traditional multi-view clustering algorithms. 

Our major contributions include: (1) The proposed method extends and improves the existing MultiNMF method by automatically computing both the view-specific and observation-specific reconstruction weights without requiring the use of prior knowledge. The two types of weights provide two different resolutions in understanding the effects of different views. Thus, the consensus matrix can be obtained by weighting different views, which efficiently extracts different information qualities from each view. 
(2) We study the properties of these two weighting schemes and provide guidance on choosing the tuning parameters.



The rest of the paper is organized as follows. In Section \ref{pre}, we introduce notations and overview existing algorithms most relevant to our proposed methods. 
In Sections \ref{model} and \ref{opt}, we present our proposed weighted multi-view NMF (WM-NMF) algorithm and study the optimization procedures.  
In Section \ref{exp}, experimental results are reported for the handwritten digit data and multi-omics biological data. Comparisons are made with some competing models and popular methods. We conclude with discussions in Section~\ref{dis}.

\section{Preliminary}
\label{pre}
Denote a nonnegative data matrix $\mathbf{X} = (\mathbf{x}_1,\dots, \mathbf{x}_N)\in\mathbb{R}^{M\times N}_+$, where $\mathbf{x}_i = (x_{1i},...,x_{Mi})^\top$ $\in \mathbb{R}_+^M$ is the $i$-th data point of $\mathbf{X}$ containing $M$ features.  	
NMF factorizes $\mathbf{X}$ into a product of two lower-dimensional nonnegative matrices: $\mathbf{X} \approx \mathbf{U}\mathbf{V}^\top$, where $\mathbf{U}\in\mathbb{R}_+^{M\times K}$, $\mathbf{V}\in\mathbb{R}_+^{N\times K}$, and $K < \min(M,N)$ is a positive integer.  
The NMF problem $\min_{\mathbf{U}, \mathbf{V}\ge0} \Vert\mathbf{X} -\mathbf{U}\mathbf{V}^\top\Vert^2_F$ with Frobenius norm is in general nonconvex and  NP-hard, but can be solved  with iterative updates that work well in many applications~\citep{lee1999learning, lee2001}.
Different from other matrix factorization techniques, NMF provides a more intuitive and interpretable understanding through the \textit {parts-based representation}: a data point can be represented by only a few activated basis elements.  Further,  $\mathbf{V}$  directly translates to data clustering by simply assigning each data point to the basis element on which it has the highest loading; that is, data point $i$ is placed in cluster $j$ if $\mathbf{V}_{i,j}$ is the largest entry in row $i$. \citet{ding2005equivalence} further shows the equivalence between NMF and $K$-means and spectral clustering.

The multi-view NMF (MultiNMF)~\citep{Liu2013} is an extension of NMF problem to integrate multiple nonnegative data matrices obtained from a common set of data points  and conducts clustering based on the low-rank representations.  Let $\{\mathbf{X}^{(1)},...,\mathbf{X}^{(n_v)}\}$ be a set of $n_v$ views of data points, with $\mathbf{X}^{(s)} \in \mathbb{R}^{M_s \times N}_+$. Without loss of generality,  we assume all the data matrices are pre-processed and transformed when necessary. The framework of MultiNMF  attempts to approximate each view $\mathbf{X}^{(s)} \approx \mathbf{U}^{(s)}{\mathbf{V}^{(s)}}^\top$ with some constraints in order to obtain both consensus  and view-specific information. 
More specifically, the MultiNMF minimizes the following objective function:
\begin{equation*}
	\label{eq:multiNMF}
	\sum_{s=1}^{n_v}\Big\Vert\mathbf{X}^{(s)} - \mathbf{U}^{(s)}{\mathbf{V}^{(s)}}^\top\Big\Vert^2_F + \sum_{s =1}^{n_v}\alpha_s\Big\Vert\mathbf{V}^{(s)}\mathbf{Q}^{(s)} - \mathbf{V}^*\Big\Vert^2_F
\end{equation*}
with respect to $\mathbf{U}^{(s)}, \mathbf{V}^{(s)}, \mathbf{V}^*\ge 0$.
The first part of the objective function performs NMF analysis independently on each view. The second part plays a key role in sharing information across views, and it regularizes the learned coefficient matrices $\mathbf{V}^{(s)}$'s towards a common $\mathbf{V}^{^*}$. We take $\mathbf{V}^{^*}$ as some latent
data structure shared by all views. The amount of information for each view contributing to $\mathbf{V}^{^*}$ is regularized by $\alpha_s$. Thus, $\alpha_s$ is the parameter that tunes the relative weight among views. 
$\alpha_s$'s have the constraints that $\sum_{s=1}^{n_v} \alpha_s = 1$ and $0 \le \alpha_s\le 1,$ $s = 1:n_v$.  
$\alpha_s$'s are crucial in determining the quality of the consensus matrix $\mathbf{V}^*$. 
The MultiNMF degenerates to the single-view learning when $\alpha_s$'s are binary values with only one component being $1$. The resulting consensus clustering is essentially determined by the view that provides the best approximation to the original data.

Most existing MultiNMF related methods tackle different problems with slightly different objective functions motivated by different applications~\citep{10.1093/nar/gks725, 10.1093/bioinformatics/bts476, 10.1371/journal.pcbi.1004042, 10.1093/bioinformatics/btv544}.  
However, most of them assume that the weight vector is determined either by prior knowledge (which may be impractical when such knowledge is missing) or assigned to be equal. In practice, such an assumption may not be valid as we often have noisy datasets. 
In Section~\ref{model}, we provide an alternative solution to allow a more interpretable and transparent understanding of how to derive the consensus clustering among views. 

\section{Weighted multi-view NMF}
\label{model}

To take the advantage of the consensus matrix used in MultiNMF and learn the weight vector automatically,  we adopt the idea of exponential parameter to automatically quantify each view's information content~\citep{tzortzis2009convex, xu2016weighted}. In addition, as demonstrated in the handwritten digit dataset and the multi-omics data for liver hepatocellular carcinoma from Section~\ref{exp}, the same data point across views is likely heterogeneous in determining the clustering structure. Thus, it is also important to determine the weight of each observation to describe the relative information content. 
This is achieved by quantifying the relative reconstruction errors for the same data point across all views.
We refer to such weight as observation-specific reconstruction weight. 
For simplicity, we call it \textit{reconstruction weight} throughout the paper.
The strategy of weighting has also been studied in the literature but in a different approach, for example, \cite{li2008weighted} weighs each input clustering.

With the abovementioned information, we propose a weighted multi-view NMF (WM-NMF) framework for a more interpretable data integration procedure, while it achieves the ability to automatically update view weight and reconstruction weight.
More specifically, WM-NMF works on minimizing the following objective function:
\begin{equation}
\label{equ:O}
\begin{aligned} 
	\mathcal{O} =&  \sum_{s=1}^{n_v}\Big\Vert\Big\{\mathbf{X}^{(s)} - \mathbf{U}^{(s)}{\mathbf{V}^{(s)}}^\top\Big\} \text{Diag}(\bm {w}^{(s)})\Big\Vert_F^2 +\\ &\sum_{s=1}^{n_v}\alpha^p_s\Big\Vert\mathbf{V}^{(s)}\mathbf{Q}^{(s)} - \mathbf{V}^*\Big\Vert^2_F+\beta g(\mb V^{(1:n_v)}),    
\end{aligned}
\end{equation}
where $g(\Vb^{(1:n_v)})$ is a regularization term on $\Vb^{(1)},\dots, \Vb^{(n_s)}$ which can be set for different purposes, {such as sparse NMF \citep{hoyer2004non}, orthogonal NMF \citep{zhang2019greedy, liang2020multi}, and graph NMF \citep{cai2010graph, huang2014robust}}. $\beta > 0$ is the corresponding tuning parameter. 
We minimize $\mathcal{O}$ over $\mathbf{U}^{(s)}, \mathbf{V}^{(s)}, \bm{w}^{(s)}, \alpha_s$, and $\mathbf{V}^*$ under the constraints that $\mathbf{V}^* \ge 0, \mathbf{V}^{(s)}\ge 0, \mathbf{U}^{(s)} \ge0,  \sum_{s=1}^{n_v}\alpha_s = 1, \alpha_s \ge 0,  \sum_{s=1}^{n_v}w_i^{(s)} = 1, w_i^{(s)} \ge 0, s = 1:n_v, i = 1:N$.

Here, $p \ge 1$ is the exponential parameter and it controls the sparsity of $\alpha_s$. We provide a discussion about $p$ in Section~\ref{alpha_est}. The vector $\bm\alpha = (\alpha_1,...,\alpha_{n_v})^\top$ represents the relative weight among different views and the agreement between $\mathbf{V}^{(s)}$ and $\mathbf{V}^*$. It reflects each view's contribution for reaching the consensus matrix $\mathbf{V}^*$.
$\bm{w}^{(s)} = (w^{(s)}_1,\dots,w^{(s)}_N)^\top$, where $w^{(s)}_i$ is the reconstruction weight of data point $i$ in view $s$. Its functionality is different from the view-specific weights, where $\alpha_s$ provides an  overall measure to quantify the contribution from view $s$ towards a consensus matrix. 
Intuitively, a relatively smaller value of ${w}^{(s)}_i$ implies that the low-dimensional representation $\mb v_i^{(s)}$ fails to reconstruct observation $\mb x_i^{(s)}$, compared to other views.
The introduction of  ${w}^{(s)}_i$ provides the flexibility to allow  one view to compensate for the shortcoming in another, and potentially prevents the spurious results from noisy or highly divergent views. 
Therefore, by constraining $\sum_{s=1}^{n_v}w_i^{(s)} = 1$ and $\sum_{s=1}^{n_v}\alpha_s = 1$, we show the feasibility to automatically update the weights as demonstrated in Section~\ref{opt}.

WM-NMF framework extends and improves the existing literature with several benefits. 	First, it automatically computes the weight vectors without relying on any prior knowledge.
Second, it calculates the consensus matrix by weighting different coefficient matrices, which efficiently extracts different qualities of information from each view.
Third, it can alleviate the negative effects of unimportant views and enlarge the positive effects of important views by assigning small and large weights on different views and observations, respectively.
Lastly, additional regularization can be easily incorporated based on our WM-NMF framework. For example, a manifold regularization can be used to further improve the clustering results~\citep{cai2010graph}.

The idea of manifold regularization is 
based on the local invariance assumption such that the geometric structure of the original dataset is inherited in the low-rank representations~\citep{belkin2001laplacian}. To extend the existing manifold regularization for single-view NMF to accommodate our multi-view NMF, we first define an adjacency matrix $\mathbf{A}^{(s)}$ to measure the closeness between any two data points represented by view $s$. 
We adopt the Gaussian kernel, $a_{ij}^{(s)}=\exp \left(-\frac{\|\mathbf{x}^{(s)}_{i}- \mathbf{x}^{(s)}_{j}\|^2_2}{\sigma^{2}}\right)$ if $\mathbf{x}^{(s)}_j \in \mathcal{N}^{(s)}_i$ and 0 otherwise,
where $\mathcal{N}^{(s)}_i$ denotes the neighbour for point $i$ represented by view $s$. $\mathcal{N}^{(s)}_i$ is generated using $K$-nearest neighbour which utilizes the distance between two data points: 
$
\|\mathbf{x}^{(s)}_{i}- \mathbf{x}^{(s)}_{j}\|^2_2$. The number of neighbours is set to be 5 and $\sigma^2=1$ as suggested in \citet{cai2010graph}.

Thus, together with the corresponding low-dimensional representation $\mathbf{v}_i^{(s)}$, the manifold regularization is defined as 
$$
S = \frac{1}{2} \sum_{i, j=1}^{N}\left\|\mathbf{v}^{(s)}_i-\mathbf{v}^{(s)}_j\right\|^{2} a_{ij}^{(s)}=\text{Tr}\left(\mathbf V^{(s)^\top} \mathbf{L}^{(s)} \mathbf V^{(s)}\right),
$$
where $\mathbf{L}^{(s)}=\mathbf{D}^{(s)}-\mathbf{A}^{(s)}$ is the graph Laplacian matrix and $\mathbf{D}$ is a diagonal matrix with the $i$th diagonal entry being $\sum_{j=1}^{N} a_{ij}^{(s)}$. $\text{Tr}(\cdot)$ denotes the trace of a matrix.  By minimizing $S$, we expect that if $\mathbf{x}_i^{(s)}$ and $\mathbf{x}_j^{(s)}$ are close, i.e., $a^{(s)}_{ij}$ is large, the corresponding low-dimensional representations $\mathbf{v}_i^{(s)}$ and $\mathbf{v}_j^{(s)}$ are also close together.

Replacing $g(\mb V^{(1:n_v)})$ in Eq.~\eqref{equ:O} by the above manifold regularization, we can define the objective function of the manifold regularized WM-NMF as 
\begin{equation}
	\label{obj}
	\begin{aligned}
		&\mathcal{O}=  \sum_{s=1}^{n_v}\Big\Vert\Big\{\mathbf{X}^{(s)} - \mathbf{U}^{(s)}{\mathbf{V}^{(s)}}^\top\Big\} \text{Diag}(\bm {w}^{(s)})\Big\Vert_F^2 +\\ &\sum_{s=1}^{n_v}\alpha^p_s\Big\Vert\mathbf{V}^{(s)}\mathbf{Q}^{(s)} - \mathbf{V}^*\Big\Vert^2_F+ 
		\beta\sum_{s=1}^{n_v} \text{Tr} (\mathbf V^{(s)^\top} \mathbf{L}^{(s)} \mathbf{V}^{(s)}).
	\end{aligned}
\end{equation}
We will discuss how to choose $\beta$ in the experiment section.
The scenario for WM-NMF without manifold regularization can be retrieved by setting $\beta=0$. 
An illustration of the manifold regularized WM-NMF is shown in Figure \ref{model_flow}.
In Section \ref{opt}, the optimization procedures are based on the objective function $\mathcal{O}$ in Eq.~\eqref{obj}.

\begin{figure}[hbt!]
	\centering
	\includegraphics[width=0.8\textwidth]{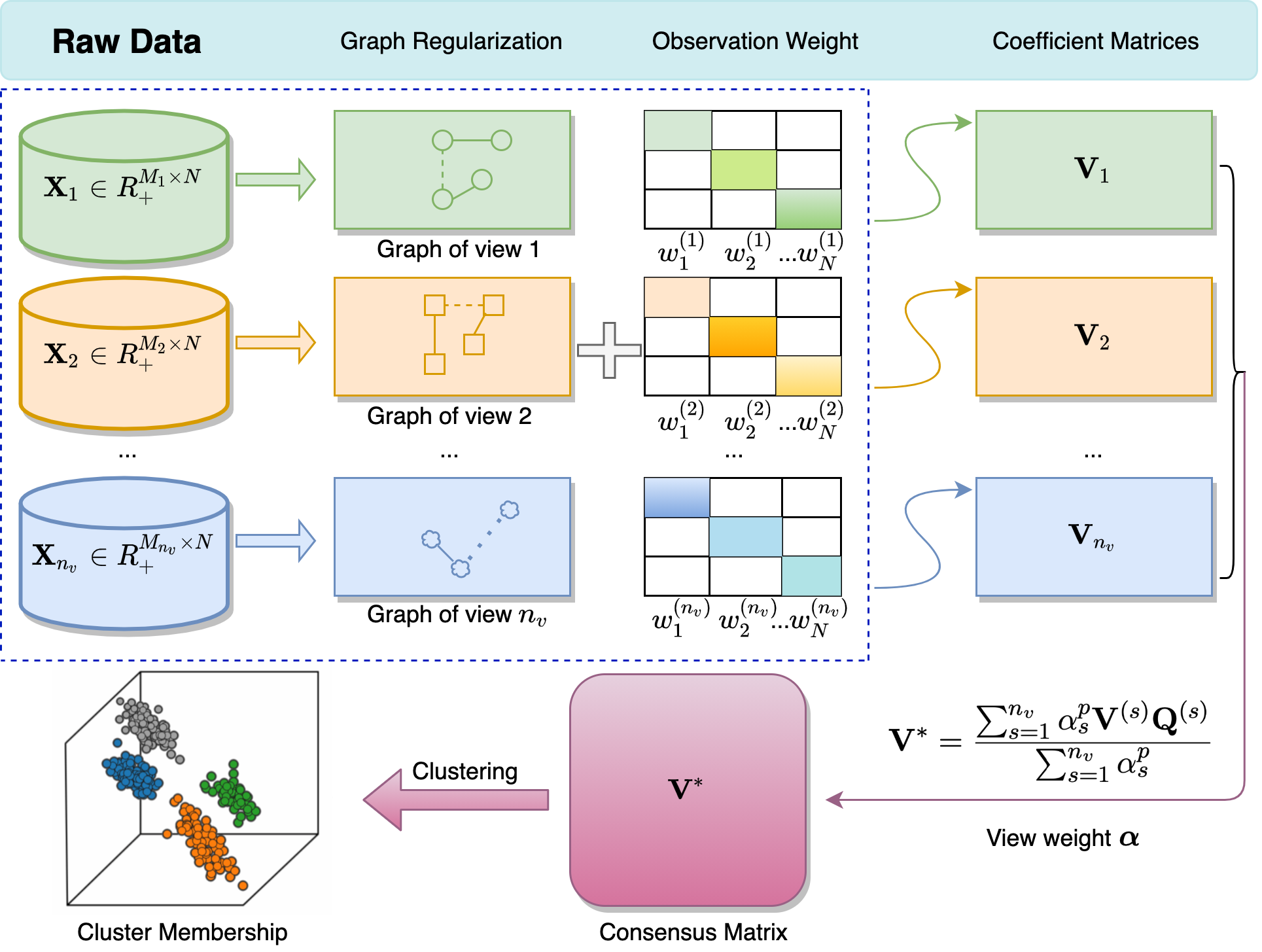}
	\caption{The illustration of the weighted multi-view NMF with manifold regularization for clustering.}
	\label{model_flow}
\end{figure}

\section{Optimization}
\label{opt}
The joint optimization function in Eq.~\eqref{obj} is nonconvex over all variables. However, if we keep four of the five variables ($\Ub^{(s)}$, $\Vb^{(s)}$,  $\Vb^*$, $\bm{\alpha}$, $\bm{w}^{(s)}$) fixed, and optimize over one of them, the problem is convex and can be solved efficiently. 
We thus consider the following iterative alternating minimization method until convergence.
At each iteration, we optimize over the five variables alternatively. 

More specifically, we first update $\Ub^{(s)}$ and $\Vb^{(s)}$ individually while keeping the others fixed. 
We call these procedures the \textit{inner iteration}.
Updating $\mb U^{(s)}$ or $\mb V^{(s)}$ only needs to solve the single-view objective function represented by
\begin{equation} \label{eql0}
	\begin{aligned}
	&\mathcal{O}_0=\Big\Vert \left\{\mb X^{(s)}-\mb U^{(s)}\mb V^{(s)^\top}\right\}\text{Diag}(\bm w^{(s)})\Big\Vert_F^2+\\
	&\alpha_s^p \Big\Vert \mb V^{(s)}\mb Q^{(s)}-\mb V^*\Big\Vert_F^2+\beta \text{Tr} (\mb V^{(s)^\top} \mb L^{(s)}\mb V^{(s)}).
	\end{aligned}
\end{equation}
After obtaining $\Ub^{(s)}$ and $\Vb^{(s)}$, we solve the exact solutions of $\bm\alpha$, $\bm w^{(s)}$ and $\mathbf{V}^*$.
We call these procedures the \textit{outer iteration}.

\subsection{Update $\Ub^{(s)}$}
\label{update_u}
To simplify the notation, we omit the view index $(s)$ for the derivations of the inner iteration.
Taking the derivative of $\mathcal{O}_0$ with respect to $u_{ij}$ while taking into account of the nonnegative constraint on $u_{ik}$, and using the complementary slackness condition $\Psi_{ik}u_{ik}=0$, where $\Psi_{ik}$ is the Lagrange multiplier for the constraint $u_{ik}\geq 0$ leads to the multiplicative update rule of $u_{ik}$:
\begin{equation*}
\label{U}
\begin{aligned}
u_{ik} &\leftarrow u_{ik} \frac{\left[\mathbf{X}\text{Diag}^2(\bm w) \Vb\right]_{ik}+\alpha_s^p\sum_{j=1}^{N} v_{jk} v_{jk}^{*}}{\left[\Ub\Vb^\top\text{Diag}^2(\bm w)\Vb\right]_{ik}+\alpha_s^p\sum_{l=1}^{M} u_{lk} \sum_{j=1}^{N} v_{jk}^{2}}\\
&=u_{ik}-\underbrace{\frac{u_{ik}}{\left[\Ub\Vb^\top\text{Diag}^2(\bm w)\Vb\right]_{ik}+\alpha_s^p\sum_{l=1}^{M} u_{lk} \sum_{j=1}^{N} v_{jk}^{2}}}_\text{step size}\times\frac{\nabla_U \mathcal{O}_0}{2},
\end{aligned}
\end{equation*}
where $\nabla_U \mathcal{O}_0=\frac{\partial \mathcal{O}_0}{\partial u_{ik}}$. Details for the derivation is in Appendix A. 
The update can be viewed as an adaptive gradient descent algorithm, where the step size should be nonzero.
Therefore, we should initialize a positive $u_{ik}$, otherwise $u_{ik}$ equals to $0$ for all subsequent iterations. 

\subsection{Update $\Vb^{(s)}$}
Similarly, we omit the view index $(s)$ for notation simplicity. 
Let $\Phi_{jk}$ be the Lagrange multiplier for the constraint $v_{jk}\geq 0$.
Setting the derivative of $\mathcal{O}_0$ with respect to $v_{jk}$ to be 0 while taking into consideration of the nonnegative constraint of $v_{jk}$, and using the complementary slackness condition $\Phi_{jk}v_{jk}=0$, we have the multiplicative update rule of $v_{jk}$:

\begin{equation*}
\label{V}
\begin{aligned}
v_{jk} &\leftarrow v_{jk} \frac{\left[\text{Diag}^2(\bm w)\mathbf{X}^\top\Ub\right]_{jk}+\alpha_s^p\left[\Vb^*\mathbf{Q}^\top\right]_{jk}+\beta\left[\mathbf{A}\Vb\right]_{jk}}{\left[\text{Diag}^2(\bm w)\Vb\Ub^\top\Ub\right]_{jk}+\alpha_s^p\left[\Vb\mathbf{Q}\mathbf{Q}^\top\right]_{jk}+\beta\left[\mathbf{D}\Vb\right]_{jk}}\\
&=v_{jk}-\underbrace{\frac{v_{jk}}{\left[\text{Diag}^2(\bm w)\Vb\Ub^\top\Ub\right]_{jk}+\alpha_s^p\left[\Vb\mathbf{Q}\mathbf{Q}^\top\right]_{jk}+\beta\left[\mathbf{D}\Vb\right]_{jk}}}_\text{step size}\times\nabla_V \mathcal{O}_0,
\end{aligned}
\end{equation*}
where $\nabla_V \mathcal{O}_0=\frac{\partial \mathcal{O}_0}{\partial v_{jk}}$, and $\mathbf{L} = \mathbf{D} - \mathbf{A}$ is the graph Laplacian matrix defined in Section~\ref{model}. 
The update can be viewed as an adaptive gradient descent algorithm, where the step size should be nonzero. Again, we should make sure the initialization of $v_{jk}$ is positive, otherwise $v_{jk}=0$ at all subsequent iterations.

Proposition \ref{uv>0} below ensures that when we initialize positive $u_{ik}$ and $v_{jk}$, the entries of $\Ub$ and $\Vb$ will always be updated as positive numbers, and the updated values will not get trapped in 0.
We provide the proof in Appendix B.
\begin{proposition}
	\label{uv>0}
	If $u_{ik}^1>0$ and $v_{jk}^1>0$, $\forall i,j,k$, then 
	$u_{ik}^t>0, v_{jk}^t>0$, $\forall i,j,k, \forall t\geq 1$, where $t$ denotes the $t$-th update. 
\end{proposition}

\subsection{Estimate ${\alpha}_s$}
\label{alpha_est}
This is equivalent to minimizing the following objective over $\alpha_s$ that
$$
\min_{\alpha_s} \alpha_s^p \Big\Vert \mathbf{V}^{(s)}\mathbf{Q}^{(s)}-\mathbf{V}^*\Big\Vert_F^2,\text{ subject to }\sum_{s=1}^{n_v}\alpha_s=1. 
$$ 
When $p=1$, the optimal solution of $\alpha_s$ is
$$\hat\alpha_s=\left\{\begin{array}{ll}
	1, & s=\underset{s'\in{1,...,n_v}}{\arg\min} \Vert \Vb^{(s')}\mathbf{Q}^{(s')}-\Vb^* \Vert_F^2 \\
	0, & \text { otherwise.}
\end{array}\right.$$
The above solution implies that $p=1$ only offers a binary solution of $\alpha_s$, i.e., the consensus matrix $\mathbf{V}^*$ depends on a single view. Such a solution is obviously too restrictive, as it  prevents the partial information sharing among views. On the other hand, when $p>1$, we obtain the optimal solution for $\bm\alpha$ as:
\begin{equation}
	\label{alpha}
	\hat\alpha_s=\frac{1}{\sum_{s'=1}^{n_v}\left(\frac{\Vert \Vb^{(s)}\mathbf{Q}^{(s)}-\Vb^*\Vert_F^2}{\Vert \Vb^{(s')}\mathbf{Q}^{(s')}-\Vb^*\Vert_F^2}\right)^{\frac{1}{p-1}}}.
\end{equation}

The solution implies that when the $s$-th view's information content contributes more to the consensus matrix, i.e., $\Vert\Vb^{(s)}\mathbf{Q}^{(s)}-\Vb^*\Vert_F^2$ is smaller, $\hat\alpha_s$ becomes larger. Therefore, the more important the view is, the larger the corresponding weight is.

\textbf{Discussion about $p$}:
Denote $A^{(s)}=\Vert\Vb^{(s)}\mathbf{Q}^{(s)}-\Vb^*\Vert_F^2$ and $A^{(s')}=\Vert\Vb^{(s')}\mathbf{Q}^{(s')}-\Vb^*\Vert_F^2$.
It is clear that as $p$ goes to infinity, 
the denominator of 
Eq.~(\ref{alpha}) converges to $n_v$, which gives
uniform weights to each view. Meanwhile, if the normalized $s$-th view $\mathbf{V}^{(s)}\mathbf{Q}^{(s)}$ contributes the most to the consensus matrix $\mathbf{V}^*$, i.e., $A^{(s)}/A^{(s')}<1$ for $s'\neq s$, then $p\rightarrow 1^+$ implies $\alpha_s\rightarrow 1$. On the other hand,
if the normalized $s$-th view $\Vb^{(s)}\mathbf{Q}^{(s)}$ contributes the least to the consensus matrix $\Vb^*$, i.e., $A^{(s)}/A^{(s')}>1$ for $s'\neq s$, then $p\rightarrow 1^+$ implies $\alpha_s\rightarrow 0$.
Hence, a smaller $p$ results in a sparser weight vector $\bm\alpha$.
Generally, a moderate size $p$ should be used so that the relevant information from different views is preserved and the effect of consensus constraint is kept.

\subsection{Estimate $w_i^{(s)}$}
To optimize ${w}^{(s)}_i$, we only consider the terms involving  ${w}^{(s)}_i$ in the objective that we consider
\begin{equation*}
	\begin{aligned}
		& \min_{\bm w^{(s)}\ge 0, \sum_{s=1}^{n_v}w_i^{(s)}=1 }\Big\Vert\Big\{\mathbf{X}^{(s)} - \mathbf{U}^{(s)}{\mathbf{V}^{(s)}}^\top\Big\} \text{Diag}(\bm w^{(s)})\Big\Vert^2_F  \\
		&= \sum_{i=1}^N {w_i^{(s)}}^2\sum_{j=1}^{M_s} {\mathbf{Y}_{ji}^{(s)}}^2,
	\end{aligned}
\end{equation*} 
where $\mathbf{Y}^{(s)} = \mathbf{X}^{(s)} - \mathbf{U}^{(s)}{\mathbf{V}^{(s)}}^\top$. Since we only optimize the weight for a single observation, it is equivalent to minimizing ${w_i^{(s)}}^2\sum_{j=1}^{M_s} {\mathbf{Y}_{ji}^{(s)}}^2$ with the constraint $\sum_{s=1}^{n_v} w^{(s)}_i = 1$.  We have that  the optimal solution is
\begin{equation}
	\label{w}
	\hat{w}^{(s)}_i = \left(\sum_{s' = 1}^{n_v}\frac{1}{\sum_{j=1}^{M_{s'}}(\mathbf{Y}^{(s')}_{ji})^2}\sum_{j=1}^{M_s}(\mathbf{Y}^{(s)}_{ji})^2\right)^{-1}.
\end{equation}

It is easy to find the solution is nonnegative.
The above solution shows that the reconstruction weight is determined by the reconstruction error of the $s$-th view on the $i$-th observation across all the features. The smaller the error is compared with other views, the larger the weight is.  
Note that ${w}^{(s)}_i$ is the weight of an observation, so the algorithm may run slowly with a very large sample size.

\subsection{Estimate $\Vb^*$}
To optimize $\Vb^*$, we only consider the terms involving $\Vb^*$ in the objective $\mathcal{O}=\sum_{s=1}^{n_v} \alpha_s^p \Big\Vert \Vb^{(s)}\mathbf{Q}^{(s)}-\Vb^*\Big\Vert_F^2$.
Setting the derivative of $\mathcal{O}_v$ with respect to $\Vb^*$ to 0, we have that the optimal solution is  
\begin{equation}
	\label{Vo1}
	\mathbf{V}^{*}=\frac{\sum_{s=1}^{n_{v}} \alpha_s^p \mathbf{V}^{(s)} \mathbf{Q}^{(s)}}{\sum_{s=1}^{n_{v}} \alpha_s^p}.
\end{equation}
Since $\Vb^{(s)}\geq 0$, $\mathbf{Q}^{(s)}\geq 0$, and $\alpha_s>0$, $\Vb^*$ is nonnegative.
The underlying assumption of the multi-view clustering is that all the views can agree and reduce to a consensus matrix with different weights, so the cluster assignments can be determined according to the consensus matrix $\Vb^*$  by the maximum coefficient assignments.
However, \citet{welch2019single} points out the spurious alignments in highly divergent datasets by the maximum coefficient assignments.
In the experiment section, we use the default function \texttt{spectralcluster} in MATLAB on $\mb V^*$ to obtain the cluster membership.

\subsection{Summary of the algorithm}
We summarize the pseudocode of the WM-NMF algorithm below.
The algorithm stops when the maximum number of iterations is reached or it converges, i.e., when the difference between the two consecutive iterations is less than the threshold $9\times 10^{-8}$.
The algorithm converges to a local minima since the objective function is nonconvex.

\begin{algorithm}[hbt!]
  {alg:gauss}%
  {\caption{Weighted Multi-View NMF (WM-NMF)}}%
{%
\textbf{Input:} {Dataset \{$\mathbf{X}^{(1)},\dots,\mathbf{X}^{(n_v)}$\}; rank $K$; exponential parameter $p$; manifold parameter $\beta$.}
	
	\textbf{Output:}{Basis matrices $\Ub^{(1)},\dots,\Ub^{(n_v)}$; Coefficient matrices $\Vb^{(1)},\dots,\Vb^{(n_v)}$; Consensus matrix $\mathbf{V}^*$; View weight vector $\alpha_s$ for each view; Reconstruction weight $w_i^{(s)}$ for each view.}
	
	Normalize each view $\mathbf{X}^{(s)}$ such that $\|\mathbf{X}^{(s)}\|_1=1$.
	
	Initialize $\Ub^{(s)}$, $\Vb^{(s)}$, $\mathbf{V}^*$, set equal view weight $\alpha_s=1/n_v$ and reconstruction weight $w^{(s)}_i=1/N$.
	
	\vspace{1mm}
	
	\textbf{repeat} \hspace{22mm} $\leftarrow$ (outer iteration)
	
	\hspace{2mm} \textbf{for} $s=1:n_v$ \textbf{do} 
	\quad $\leftarrow$ (inner iteration)
	
	\hspace{5mm} \textbf{repeat}
	
	\hspace{8mm} Fixing $\bm w^{(s)}$, $\bm\alpha$, $\mathbf{V}^*$ and $\Vb^{(s)}$, update $\Ub^{(s)}$ by Eq. (\ref{U});
	
	\hspace{8mm} Fixing $\bm w^{(s)}$, $\bm\alpha$, $\Vb^*$ and $\Ub^{(s)}$, update $\Vb^{(s)}$ by Eq. (\ref{V});
	
	\hspace{5mm} \textbf{until} $\mathcal{O}_0$ converges or the maximum number of iteration is reached. 
	
	\hspace{2mm} \textbf{end for}
	
	\hspace{2mm} Fixing $\bm w^{(s)}$, $\Ub^{(s)}$, $\Vb^{(s)}$, $s = 1,...,n_v$, and $\Vb^*$, update $\bm\alpha$ by Eq. (\ref{alpha}); 
	
	\hspace{2mm} Fixing $\bm\alpha$, $\Ub^{(s)}$, $\Vb^{(s)}$, and $\Vb^*$, update $\bm w$ by Eq. (\ref{w}); 
	
	\hspace{2mm} Fixing $\bm w^{(s)}$, $\Ub^{(s)}$, $\Vb^{(s)}$, $s = 1,...,n_v$, and $\bm\alpha$, update $\Vb^*$ by Eq. (\ref{Vo1});
	
	\textbf{until} the  maximum number of iteration is reached or the algorithm converges.
	\vspace{1mm}
	
	Perform clustering analysis based on $\Vb^*$.
	\label{algorithm1}
}%
\end{algorithm}

Since the objective function is nonconvex, the solution may depend on the initialization. 
We initialize $\Ub^{(s)}$ and $\Vb^{(s)}$ by the Graph Regularized Non-negative Matrix Factorization (GNMF) \citep{cai2010graph}.

\begin{theorem}
	\label{conv}
	The objective function $\mathcal{O}$ converges to a local minima under Algorithm~\ref{algorithm1}.
\end{theorem}

The proof is given in Appendix B. We verify Theorem \ref{conv} through different datasets and provide complexity analysis in Appendix C.
Besides, we include the complexity analysis (operation counts) in Appendix D.

Discussion of the tuning parameters: A larger value of $K$  better approximates the original data, but on the other hand, it raises the risk of overfitting. Many approaches have been developed to select the number of basis elements ${K}$ for the NMF problem, such as Bi-cross validation \citep{owen2009bi}, Stein's unbiased risk estimator \citep{ulfarsson2013tuning}, minimum description length (MDL) \citep{squires2017rank}, and missing data imputation \citep{lin2020optimization}. 
Overall, all these methods show the capacity of selecting $K$ in certain datasets empirically.
In this paper, we assume prior information of $K$ is given.
In Appendix C, we show that WM-NMF  works considerately well within a range of $K$ in terms of the clustering accuracy. Thus, it is quite robust to the choice of $K$.
Empirical results on other tuning parameters are also included in Appendix C.

\section{Experiments}
\label{exp}
In this section,  we present experimental results on one handwritten digit dataset and one multi-omics dataset. For each dataset, we use six metrics to evaluate the clustering performance: accuracy (ACC), normalized mutual information (NMI), Precision, Recall, F-score, Adjusted Rand index (Adj-RI).
For all these metrics, higher values indicate better clustering performance.
Details and formulas of them are available in \citet{manning2008introduction}. 
Empirical studies on the tuning parameters are given in Appendix C.

In addition, we compare  WM-NMF  with several competing multi-view clustering algorithms described below:
\begin{enumerate}
	\itemsep0em
	\item $K$-means:
	The default \texttt{kmeans} function in MATLAB is implemented to obtain the results. There are two strategies: (1) Apply $K$-means independently on each single view, and select the best  performance of $K$-means as the final results. We denote this strategy as BSV-kmeans. (2) Apply $K$-means on the data where all the views are concatenated. We denote this strategy as  
	ConcatK.
	
	\item Spectral clustering: The classical spectral clustering algorithm is applied to the datasets. 
	The default function \texttt{spectralcluster} in MATLAB is implemented to obtain the results. Similar to the above $K$-means, we denote
	BSV-Spectral as the best performance of spectral clustering over each single view, and 
	ConcatSpectral represents the result of spectral clustering on the data with all views concatenated.
	
	\item MultiNMF: Multi-view nonnegative matrix factorization with equal weight.
	MultiNMF1 is implemented with equal weights summing to 1, i.e.,  $\alpha_s=1/n_v$ for $s = 1,...,n_v$.
	MultiNMF2 is implemented with equal weight such that $\alpha_s = 0.01$, which is shown to have the best performance in \cite{Liu2013}. The clustering is performed based on $\Vb^*$ by $K$-means.
	
	\item MLRSSC: Multi-View low-rank sparse subspace clustering, which is shown to be very competitive in ~\citep{brbic2018multi}. More specifically, we implement four variants: P-MLRSSC, C-MLRSSC, P-KMLRSSC and C-KMLRSSC which represent pairwise, centroid-based, pairwise kernel, and centroid-based kernel multi-view low-rank sparse subspace clustering, respectively.
	
	\item NMF-W1 and NMF-W2: They are the  simplified version of WM-NMF.
We denote NMF-W1 as the case when $\bm{w}$ is fixed a priori in WM-NMF.  Comparing NMF-W1 to WM-NMF, we emphasize the importance of the reconstruction weight.
	Likewise, we denote NMF-W2 as the situation when $\bm{w}$ is fixed and $\beta$ is set to 0, such that there is no manifold regularization in WM-NMF.
	Comparing NMF-W2 to WM-NMF, we emphasize the importance of the reconstruction weight and the manifold regularization.
\end{enumerate}

For all the  experiments,  we set $p=5$ and $\beta=0.01$ for WM-NMF and we use the default settings for all the other algorithms. 

\subsection{Data description}

We summarize the key information of all datasets in Table~\ref{table2} and provide the data descriptions in Appendix C.

\begin{table}[hbt!]
	\caption{Detailed information about the experiment datasets.}\label{table2}
	\centering
	\begin{tabular}{|c|cccc|}
		\hline
		\textbf{Dataset}                      & \multicolumn{1}{c}{\textbf{view}} & \multicolumn{1}{l}{\textbf{observations}} & \multicolumn{1}{l}{\textbf{features}} & \textbf{clusters} \\ \hline
		\multirow{6}{*}{Synthetic} & $\mathbf{X}^{(1)}$ &5000 &100 &10         \\
		& $\mathbf{X}^{(2)}$ &5000 &150 &10         \\
		& $\mathbf{X}^{(3)}$ &5000 &50 &10         \\
		& $\mathbf{X}^{(4)}$ &5000 &200 &10         \\ 
		& $\mathbf{X}^{(5)}$ &5000 &100 &-         \\
		& $\mathbf{X}^{(6)}$ &5000 &50 &-        \\ \hline
		\multirow{4}{*}{Handwritten} & fou &2000 &76 &10         \\
		& pix &2000 &240 &10         \\
		& zer &2000 &47 &10         \\
		& fac &2000 &216 &10      \\ \hline
		\multirow{3}{*}{LIHC} 
		& GE &404 &15397 &2        \\
		& CNA &404 &16384  &2         \\
		& DNAm &404 &16384  &2         \\\hline
	\end{tabular}
\end{table}

\subsection{Results on the handwritten digit dataset}
Table \ref{res} compares the result of WM-NMF to the other algorithms based on the handwritten digit dataset.
It is worth noting that WM-NMF obtains the highest scores for all six evaluation metrics. All the other competing methods show clustering performance significantly worse than WM-NMF.   
Now, we analyze the effectiveness of the manifold regularization, view-specific weight and reconstruction weight, respectively.
First, we observe that NMF-W1 performs better than NMF-W2,  {which implies the importance of the manifold regularization for clustering analysis.}
Second, WM-NMF, NMF-W1, and NMF-W2 all outperform MultiNMF, this shows {the ineffectiveness of using equal weight for $\bm \alpha$}.
Third, we find that WM-NMF hits higher scores than NMF-W1. This demonstrates the advantage of using reconstruction weight and the ability of WM-NMF on integrating heterogeneous data.


\begin{table*}[htb]
	\caption{Comparisons of clustering performance between WM-NMF and other competing methods for handwritten digit dataset. Numbers in the bracket represent standard deviations of the corresponding scores, which is obtained based on 20 replications for each algorithm.}	\label{res}
	
	\resizebox{\textwidth}{!}{%
		\begin{tabular}{|l|l|l|l|l|l|l|}
			\hline\multicolumn{1}{|c|}{Algorithm} & \multicolumn{1}{|c|}{ACC} & \multicolumn{1}{|c|}{NMI} & \multicolumn{1}{|c|}{Precision} & \multicolumn{1}{|c|}{Recall} & \multicolumn{1}{|c|}{F-score} & \multicolumn{1}{|c|}{Adj-RI}\\ \hline
			 BSV-kmeans & 0.69 (0.07) & 0.70 (0.03) & 0.61 (0.05) & 0.67 (0.04) & 0.63 (0.05) & 0.59 (0.05) \\ 
			 ConcatK   & 0.63 (0.07) & 0.62 (0.03) & 0.51 (0.05) & 0.59 (0.03) & 0.55 (0.04) & 0.49 (0.04) \\ 
			BSV-Spectral   & 0.68 (0.00) & 0.71 (0.00) & 0.58 (0.00) & 0.68 (0.00) & 0.62 (0.00) & 0.58 (0.00) \\ 
			ConcatSpectral   & 0.12 (0.00) & 0.01 (0.00) & 0.10 (0.00) & 0.41 (0.04) & 0.16 (0.00) & 0.00 (0.00) \\ 
			MultiNMF1   & 0.64 (0.03) & 0.58 (0.02) & 0.51 (0.03) & 0.54 (0.03) & 0.52 (0.03) & 0.47 (0.03) \\ 
			MultiNMF2   & 0.79 (0.04) & 0.72 (0.02) & 0.66 (0.03) & 0.69 (0.03) &0.68 (0.03) & 0.64 (0.03) \\ 
			P-MLRSSC   & 0.75 (0.07) & 0.77 (0.04) & 0.68 (0.07) & 0.75 (0.05) & 0.71 (0.06)& 0.68 (0.07) \\ 
			C-MLRSSC & 0.75 (0.06) & 0.77 (0.04) & 0.68 (0.06) & 0.74 (0.05) & 0.71 (0.06)& 0.67 (0.06) \\ 
			P-KMLRSSC   & 0.77 (0.06) & 0.72 (0.02) & 0.66 (0.05) & 0.68 (0.05) & 0.67 (0.04)& 0.63 (0.05) \\ 
			C-KMLRSSC & 0.76 (0.07) & 0.72 (0.03) & 0.65 (0.06) & 0.68 (0.05) & 0.67 (0.05)& 0.63 (0.06) \\ 
			NMF-W1   & 0.92 (0.03) & 0.88 (0.03) & 0.85 (0.06) & 0.88 (0.03) & 0.86 (0.05) & 0.84 (0.05) \\ 
			NMF-W2   & 0.81 (0.08) & 0.77 (0.05) & 0.69 (0.10) & 0.76 (0.06) & 0.72 (0.08) & 0.69 (0.09) \\ 
			WM-NMF   & \textcolor{black}{0.96 (0.02)} & \textcolor{black}{0.93 (0.01)} & \textcolor{black}{0.93 (0.04)} & \textcolor{black}{0.94 (0.01)} & \textcolor{black}{0.93 (0.03)} & \textcolor{black}{0.93 (0.03)} \\
			\hline
	\end{tabular}}
\end{table*}

\subsection{Results on the multi-omics LIHC dataset}
\citet{seal2020estimating} uses both CNV and DNAm to predict {the sample status, either tumor or normal sample,} so it is treated as a classification problem and the accuracy is 95.1\%.
In this paper, we treat this data as unsupervised problem and integrate all the three omics data to conduct clustering analysis. 

Figure \ref{LIHC_result} presents the analysis results.
As it shows, both WM-NMF and NMF-W2 outperform the other algorithms in all the evaluation metrics while WM-NMF behaves much better than NMF-W2.
The average scores of the proposed WM-NMF for the 6 evaluation metrics are 0.97, 0.70, 0.99, 0.95, 0.97, and 0.83, respectively.
Note that we do not plot the results of MLRSSC algorithms due to their code constraints.
Instead, we report the highest six metric scores with standard deviations among the four algorithms: 0.57 (0.04), 0.18 (0.00), 0.88 (0.00), 0.54 (0.00), 0.67 (0.00) and 0.10 (0.00).
Besides, the proposed WM-NMF algorithm has lower standard deviations compared to all the other algorithms.
It is worth noting that the clustering algorithm WM-NMF achieves slightly higher accuracy than the neural network approach for classification in \citet{seal2020estimating}.
\begin{figure}[hbt!]
	\centering
	\includegraphics[width=0.8\textwidth]{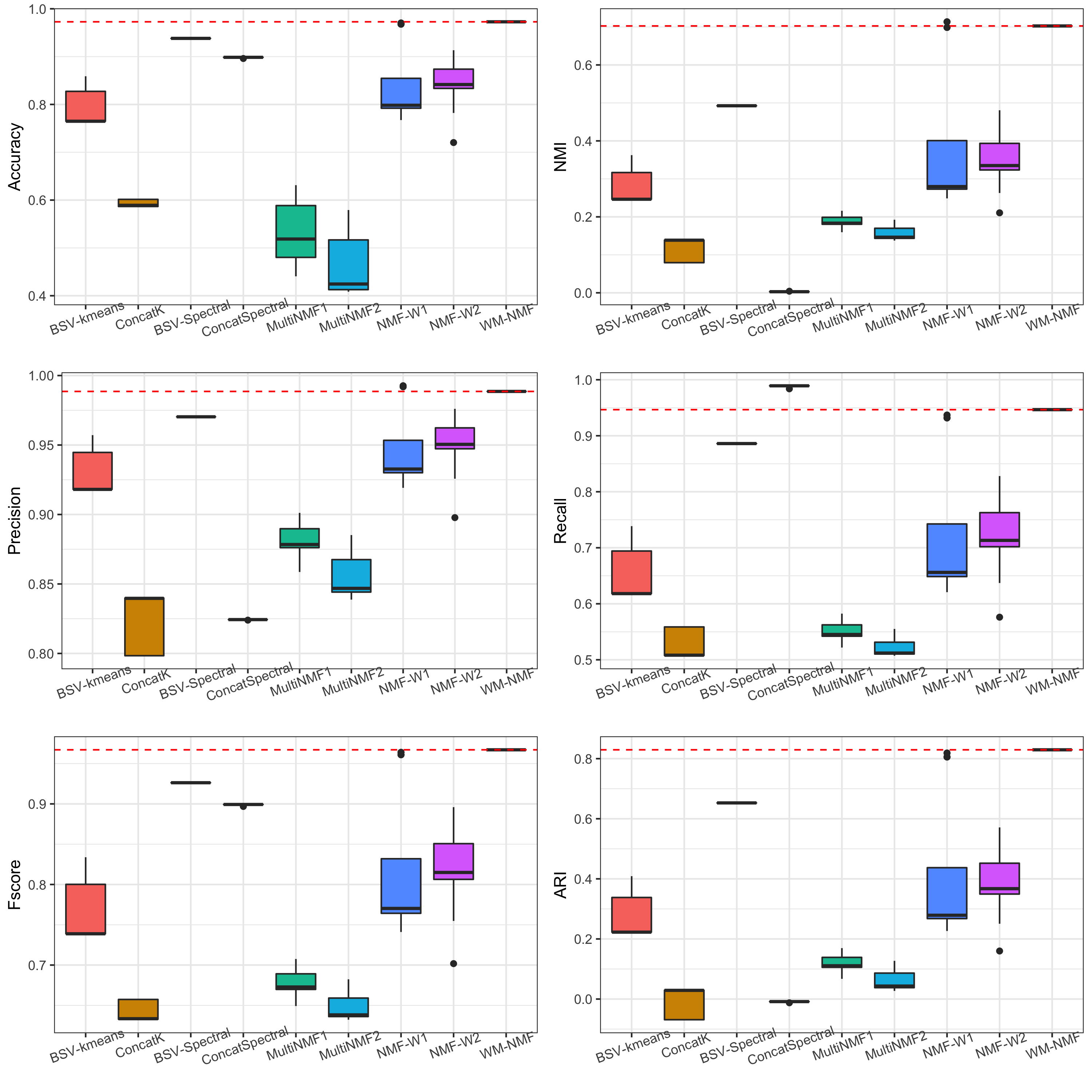}
	\caption{Boxplots representing clustering results for LIHC dataset based on 20 replications for each algorithm. 
		The dashed red line is the average score of WM-NMF. }
	\label{LIHC_result}
\end{figure}

\section{Discussion}
\label{dis}
We develop a weighted multi-view NMF (WM-NMF) algorithm, with the goal of learning multi-view data for integrative clustering analysis.
One key feature of WM-NMF is the ability to learn both view-specific and reconstruction weights to quantify each view’s information content. Thus, the unnecessary views’ adverse effects can be alleviated and the positive effects of the important views are enlarged, making WM-NMF robust to the potentially heterogeneous multi-view data. 
Such ability enables WM-NMF to deal with heterogeneous and noisy data. 
Technically, our proposed weighting scheme can be integrated into other methods such as model-based approaches. 
Therefore, we may combine the benefits of the model-based approaches with the weighting scheme to study the theoretical properties.


\appendix
\clearpage
\onecolumn
\section*{Appendix A}
In this section, we provide the derivations for updates of $\mathbf U^{(s)}$, $\mathbf V^{(s)}$, and $\alpha_s$. 

\subsection*{Update of $\mathbf  U^{(s)}$ and $\mathbf  V^{(s)}$}
With the notations in the main text, we have
\begin{equation*}
	\label{der_U}
	\frac{\partial \mathcal{O}_1}{\partial u_{ik}}=2\left[\Ub\Vb^\top\text{Diag}^2(\bm w)\Vb\right]_{ik}-2\left[\mathbf{X}\text{Diag}^2(\bm w) \Vb\right]_{ik}+\alpha_s^p P_{ik}+\Psi_{ik},
\end{equation*}
where $P_{ik}=2\left(\sum_{l=1}^{M} u_{lk} \sum_{j=1}^{N} v_{jk}^{2}-\sum_{j=1}^{N} v_{jk} v_{jk}^{*}\right)$ and $\Psi_{ik}$ is the Lagrange multiplier for the constraint $u_{ik}\geq 0$.
Using the complementary slackness condition $\Psi_{ik}U_{ik}=0$, plugging the expression $P_{ik}$ into Eq. (\ref{der_U}) and setting it to 0, we have 
$$
\left[\Ub\Vb^\top\text{Diag}^2(\bm w)V\right]_{ik}u_{ik}-\left[\mathbf{X}\text{Diag}^2(\bm w) \Vb\right]_{ik}u_{ik}+\alpha_s^p\left(\sum_{l=1}^{M} u_{lk} \sum_{j=1}^{N} v_{jk}^{2}-\sum_{j=1}^{N} v_{jk} v_{jk}^{*}\right)u_{ik}=0
$$
$$\Leftrightarrow \left[\Ub\Vb^\top\text{Diag}^2(\bm w)V\right]_{ik}u_{ik}+\alpha_s^p\sum_{l=1}^{M} u_{lk} \sum_{j=1}^{N} v_{jk}^{2}u_{ik}=\left[\mathbf{X}\text{Diag}^2(\bm w) \Vb\right]_{ik}u_{ik}+\alpha_s^p\sum_{j=1}^{N} v_{jk} v_{jk}^{*}u_{ik}.$$

Similarly for $\mathbf V^{(s)}$, we have
$$
\frac{\partial \mathcal{O}_2}{\partial v_{jk}}=\left[\text{Diag}^2(\bm w)\Vb\Ub^\top\Ub\right]_{jk}-\left[\text{Diag}^2(\bm w)\mathbf{X}^\top\Ub\right]_{jk}+\alpha_s^p\left[\Vb\mathbf{Q}\mathbf{Q}^\top-\Vb^*\mathbf{Q}^\top\right]_{jk}+\beta\left[\mathbf{L}\Vb\right]_{jk}+\Phi_{jk}=0
$$
$$\downdownarrows$$
$$\left[\text{Diag}^2(\bm w)\Vb\Ub^\top\Ub\right]_{jk}+\alpha_s^p\left[\Vb\mathbf{Q}\mathbf{Q}^\top\right]_{jk}+\beta\left[D\Vb\right]_{jk}+\Phi_{jk}=\left[\text{Diag}^2(\bm w)\mathbf{X}^\top\Ub\right]_{jk}+\alpha_s^p\left[\Vb^*\mathbf{Q}^\top\right]_{jk}+\beta\left[\mathbf{A}\Vb\right]_{jk},$$
Multiplying the above equation by $v_{jk}$ and using the complementary slackness condition $\Phi_{jk}V_{jk}=0$ gives the result.

\subsection*{Update of $\alpha_s$}	
When $p>1$, setting the derivative of $\mathcal{O}$ that only contains $\alpha_s$ with respect to $\alpha_s$ to 0, we get
$$p\alpha_s^{(p-1)}A+\lambda_1=0\quad\Rightarrow\quad \alpha_s=\left(-\frac{\lambda_1}{pA}\right)^{\frac1{p-1}},$$
where we assume $A^{(s)}=\| \Vb^{(s)}\mathbf{Q}^{(s)}-\Vb^*\|_F^2>0$.
Given the constraint that $\sum_{s'=1}^{n_v}\alpha_{s'}=1$, we have
$$\sum_{s'=1}^{n_v}\left(-\frac{\lambda_1}{pA^{(s')}}\right)^{\frac1{p-1}}=1\quad\Rightarrow\quad \left(-\lambda_1\right)^{\frac1{p-1}}=\frac{1}{\sum_{s'=1}^{n_v}\left(\frac{1}{pA^{(s')}}\right)^{\frac{1}{p-1}}},$$
Finally, we obtain the solution of $\alpha_s$
\begin{equation*}
	\hat\alpha_s=\frac{1}{\sum_{s'=1}^{n_v}\left(\frac{A^{(s)}}{A^{(s')}}\right)^{\frac{1}{p-1}}}=\frac{1}{\sum_{s'=1}^{n_v}\left(\frac{\Vert \Vb^{(s)}\mathbf{Q}^{(s)}-\Vb^*\Vert_F^2}{\Vert \Vb^{(s')}\mathbf{Q}^{(s')}-\Vb^*\Vert_F^2}\right)^{\frac{1}{p-1}}}.
\end{equation*}

\section*{Appendix B}
In this section, we provide the proofs for Proposition \ref{uv>0} and Theorem \ref{conv}. 

\subsection*{Proof of Proposition \ref{uv>0}}
	First we assume that the denominators in Eq.~(\ref{U}) and (\ref{V}) are always well-defined. The update rules by \cite{lee2001} are not well-defined if the denominators are 0.
		This may happen in a very rare case when all the terms on the denominators are 0.
		In such case, a small positive number can be added to avoid 0 \citep{lin2007convergence}.
		When it is added, the analyses keep the same, so we stick to the situations without the small positive number in this paper.
	
	When $t=1$, Theorem~\ref{uv>0} holds by the assumption of this theorem.
	For $t>1$, we prove by induction. We first prove for the case of $\Ub$. 
	Assuming the results are true at $t$th iteration, we note that from $t$ to $t+1$, the step size for updating $u_{ik}$ in Eq.~\eqref{U} is positive:
	$$\frac{u_{ik}}{\left[\Ub\Vb^\top\text{Diag}^2(\bm w)\Vb\right]_{ik}+\alpha_s^p\sum_{l=1}^{M} u_{lk} \sum_{j=1}^{N} v_{jk}^{2}}>0.$$
	We now consider two situations for the derivative $\nabla_U \mathcal{O}_0$:\\
	Case 1: When $\nabla_U \mathcal{O}_0=0$, $u_{ik}^{t+1}=u_{ik}^t$ and it converges as the complementary slackness condition suggests.\\
	Case 2: When $\nabla_U \mathcal{O}_0\neq0$, 
	\begin{equation*}
		\begin{aligned}
			u_{ik} &\leftarrow u_{ik}-\frac{u_{ik}}{\left[\Ub\Vb^\top\text{Diag}^2(\bm w)\Vb\right]_{ik}+\alpha_s^p\sum_{l=1}^{M} u_{lk} \sum_{j=1}^{N} v_{jk}^{2}}\times\frac12\nabla_U \mathcal{O}_0\\
			&=u_{ik}-u_{ik}\frac{\frac12\nabla_U \mathcal{O}_0}{\left[\Ub\Vb^\top\text{Diag}^2(\bm w)\Vb\right]_{ik}+\alpha_s^p\sum_{l=1}^{M} u_{lk} \sum_{j=1}^{N} v_{jk}^{2}}\\
			&=u_{ik}-u_{ik}\left(1-\frac{\left[\mathbf{X}\text{Diag}^2(\bm w) \Vb\right]_{ik}+\alpha_s^p\sum_{j=1}^{N} v_{jk} v_{jk}^{*}}{\left[\Ub\Vb^\top\text{Diag}^2(\bm w)\Vb\right]_{ik}+\alpha_s^p\sum_{l=1}^{M} u_{lk} \sum_{j=1}^{N} v_{jk}^{2}}\right)\\
			&>0.
		\end{aligned}
	\end{equation*}
	The inequality follows the definition of $\nabla_U \mathcal{O}_0$.
	When $\nabla_U \mathcal{O}_0>0$, the term in the bracket is between 0 and 1.
	When $\nabla_U \mathcal{O}_0<0$, the term in the bracket is negative.
	Both imply that $u_{ik}^{t+1}>0$.
	The proof for $v_{jk}^t, \forall t\geq1$ is the same as the proof for $\Ub$ and we omit it here.

Prior to the details of proving Theorem \ref{conv}, we introduce a lemma that is essential to the proof.
\begin{lemma}[\citet{lee2001}]
	\label{lee2001}
If $G(h, h')$ is an auxiliary function of $J(h)$, then $J(h)$ is nonincreasing under the update rule	
\begin{equation}
	\label{nonin}
h^{(t+1)}=\underset{h}{\operatorname{argmin}}  G(h,h^{(t)}).
\end{equation}
\end{lemma}

\subsection*{Proof of Theorem \ref{conv}}
	The updates for $\Vb^*$ and $\alpha_s$ give exact solutions for the minimization of $\mathcal{O}$ when others are fixed.
	Therefore, we only need to prove that $\mathcal{O}$ is nonincreasing under the update rules of $\Ub^{(s)}$ and $\Vb^{(s)}$, $s = 1,...,n_v$. Again to ease the notation without confusion, we drop $(s)$ from the notations, and we simply write $\Vb$ and $\Ub$ to refer to a specific view.

	The proof is established by defining an auxiliary function and showing the Taylor-expansion of the objective function is less than or equal to the auxiliary function.
	The update rules are element-wise, and we only need to show $L_{ik}$ and $J_{jk}$ are nonincreasing for Equations (\ref{U}) and (\ref{V}), where $L_{ik}$ and $J_{jk}$ denote the part of $\mathcal{O}$ relative to $u_{ik}$ and $v_{jk}$ only, respectively.
	They are the same as $\mathcal{O}_1$ and $\mathcal{O}_2$ as defined above.
	For $u_{ik}$, if we define the function
	$$
	\begin{aligned}
		& G\left(u_{ik}, u_{ik}^{t}\right)=L_{ik}\left(u_{ik}^{t}\right)+L_{ik}'\left(u_{ik}^{t}\right)\left(u_{ik}-u_{ik}^{t}\right) \\
		+&\left\{\frac{\left(\Ub^{t} \Vb^{T}\text{Diag}^2(\bm w) \Vb\right)_{ik}}{u_{ik}^{t}}+\frac{\alpha_{s}^p \sum_{i=1}^{M} u_{ik}^{t} \sum_{j=1}^{N} v_{jk}^{2}}{u_{ik}^{t}}\right\}\left(u_{ik}-u_{ik}^{t}\right)^{2},
	\end{aligned}
	$$
	then we have $G(u_{ik}, u_{ik})=L_{ik}(u_{ik})$. 
	
	Next, we need to show $G(u_{ik}, u_{ik}^t)\geq L_{ik}(u_{ik})$.
	The Taylor expansion of $L_{ik}\left(u_{ik}\right)$ gives
	$$
	L_{ik}\left(u_{ik}\right) =L_{ik}\left(u_{ik}^{t}\right)+L_{ik}'\left(u_{i k}^{t}\right)\left(u_{ik}-u_{i k}^{t}\right)+\frac{1}{2} L_{ik}''\left(u_{i k}^{t}\right)\left(u_{ik}-u_{ik}^{t}\right)^{2},
	$$
	with the second order derivative $L_{ik}''\left(u_{ik}\right)=2\left[\Vb^{T} \text{Diag}^2(\bm w)\Vb\right]_{kk}+2 \alpha_s^p \sum_{j=1}^{N} v_{jk}^{2}$.
	Comparing the Taylor-expansion of $L_{ik}\left(u_{ik}\right)$ to $G\left(u_{ik}, u_{ik}^{t}\right)$, we only need to show 
	$$
	\frac{\left\{\Ub^{t} \Vb^{T} \text{Diag}^2(\bm w)\Vb\right\}_{ik}}{u_{ik}^{t}}+\frac{\alpha_{s}^p \sum_{i=1}^{M} u_{ik}^{t} \sum_{j=1}^{N} v_{jk}^{2}}{u_{ik}^{t}}\geq \left\{\Vb^{T} \text{Diag}^2(\bm w)\Vb\right\}_{kk}+ \alpha_s^p \sum_{j=1}^{N} v_{jk}^{2}.$$
	This is easy to verify by comparing the first and second terms of the above inequality, respectively.
	We have, according to the nonnegative constraints on $\Ub$ and $\Vb$	$$\left\{\Ub^{t} \Vb^{T} \text{Diag}^2(\bm w)\Vb\right\}_{ik}=\sum_l^K u_{il}^t \left[\text{Diag}^2(\bm w)\Vb\right]_{lk}\geq u_{ik}^{t}\left\{\Vb^{T} \text{Diag}^2(\bm w)\Vb\right\}_{kk},$$
	$$\alpha_{s}^p \sum_{i=1}^{M} u_{ik}^{t} \sum_{j=1}^{N} v_{jk}^{2}\geq\alpha_{s}^p u_{ik}^{t} \sum_{j=1}^{N} v_{jk}^{2}.$$
	Thus, $G(u_{ik}, u_{ik}^t)$ is an auxiliary function of $L_{ik}(u_{ik})$.
	Replacing $G(h,h^t)$ in Eq. (\ref{nonin}) by $G\left(u_{ik}, u_{ik}^{t}\right)$, we have 
	\begin{equation*}
		\begin{aligned}
			u_{ik}^{t+1}&= u_{ik}^t- u_{ik}^t\frac{L_{ik}'\left(u_{i k}^{t}\right)}{2\left[\Ub\Vb^\top\text{Diag}^2(\bm w)\Vb\right]_{ik}+2\alpha_s^p\sum_{l=1}^{M} u_{lk} \sum_{j=1}^{N} v_{jk}^{2}} \\
			&=u_{ik}^t\frac{\left[X\text{Diag}^2(\bm w) v\right]_{ik}+\alpha_s^p\sum_{j=1}^{N} v_{jk} v_{jk}^{*}}{\left[\Ub\Vb^\top\text{Diag}^2(\bm w)\Vb\right]_{ik}+\alpha_s^p\sum_{l=1}^{M} u_{lk} \sum_{j=1}^{N} v_{jk}^{2}}.
		\end{aligned}
	\end{equation*}
	The result follows Lemma 1 in \citet{lee2001} that $L_{ik}$ is nonincreasing under the iteration $h^{t+1}=\arg\min_h G(h, h^t)$. 
	Since the objective function is bounded below by 0, the monotone convergence theorem implies the convergence.
	
	Similar statements for the proof of $v_{jk}$ can be established by defining the auxiliary function 
	$$
	\begin{aligned}
		&G(v_{jk}, v_{jk}^{t}) =J_{jk}(v_{jk}^{t})+J_{jk}'(v_{jk}^{t})(v_{jk}-v_{jk}^{t})\\ +&\left\{\frac{\left[\text{Diag}^2(\bm w)\Vb^t \Ub^\top\Ub\right]_{jk}+\alpha_s^p\left[\Vb^t\mathbf{Q}\mathbf{Q}^\top\right]_{jk}+\beta\left[\mathbf{D}\Vb^t\right]_{jk}}{v_{jk}^{t}}\right\}(v_{jk}-v_{jk}^{t})^{2}.
	\end{aligned}$$
	It is easy to see $G(v_{jk}, v_{jk})=J_{jk}(v_{jk})$ and the remaining part is to show $G(v_{jk}, v_{jk}^t)\geq J_{jk}(v_{jk})$. 
	The Taylor-expansion of $J_{jk}(v_{jk})$ gives
	$$
	J_{jk}\left(v_{jk}\right) =J_{jk}\left(v_{jk}^{t}\right)+J_{jk}'\left(v_{j k}^{t}\right)\left(v_{jk}-v_{jk}^{t}\right)+\frac{1}{2} J_{jk}''\left(v_{jk}^{t}\right)\left(v_{jk}-v_{jk}^{t}\right)^{2},
	$$
	with the second order derivative $J_{jk}''(v_{jk})=2\left[\text{Diag}^2(\bm w)\right]_{jj}\left[\Ub^\top\Ub\right]_{kk}+2\alpha_s^p\mathbf{Q}\mathbf{Q}^\top+2\beta \left[\mb L\right]_{jj}$ (note that $\mb L$ is the graph Laplacian matrix defined in section \ref{model}).
	Comparing the Taylor-expansion of $J_{jk}(v_{jk})$ to $G(v_{jk}, v_{jk}^{t})$, we are left to show
	$$\left\{\frac{\left[\text{Diag}^2(\bm w)\Vb^t \Ub^\top\Ub\right]_{jk}+\alpha_s^p\left[\Vb^t\mathbf{Q}\mathbf{Q}^\top\right]_{jk}+\beta\left[\mathbf{D}\Vb^t\right]_{jk}}{v_{jk}^{t}}\right\}$$
	$$\geq \left[\text{Diag}^2(\bm w)\right]_{jj}\left[\Ub^\top\Ub\right]_{kk}+\alpha_s^p\mathbf{Q}\mathbf{Q}^\top+\beta\left[L\right]_{jj}.$$
	This can be verified by comparing the first and third terms of the inequality, respectively.
	We have, according to the nonnegative constraints on $\Ub$ and $\Vb$,	$$\left[\text{Diag}^2(\bm w)\Vb^t \Ub^\top\Ub\right]_{jk}=\sum_l^K v^t_{jl}\left[\text{Diag}^2(\bm w)\right]_{jj}\left[\Ub^\top\Ub\right]_{lk} \geq v_{jk}^{t}\left[\text{Diag}^2(\bm w)\right]_{jj}\left[\Ub^\top\Ub\right]_{kk},$$
	$$\beta\left[\mathbf{D}\Vb^t\right]_{jk}=\beta\sum_{l=1}^M d_{jl}v^t_{lk}\geq\beta d_{jj}v^t_{jk}\geq \beta\left[\mathbf{D}-\mathbf{A}\right]_{jj}v^t_{jk}=\beta L_{jj} v^t_{jk}.$$
	Therefore, $G(v_{jk}, u_{jk}^{t})$ is an auxiliary function of $J_{jk}$.
	Replacing $G(h,h^t)$ in Eq. (\ref{nonin}) by $G(v_{jk}, v_{jk}^{t})$, we have
	\begin{equation*}
		\begin{aligned}
			v_{jk}^{t+1}&= v_{jk}^t- v_{jk}^t\frac{J_{jk}'\left(v_{j k}^{t}\right)}{\left[\text{Diag}^2(\bm w)\Vb\Ub^\top\Ub\right]_{jk}+\alpha_s^p\left[\Vb\mathbf{Q}\mathbf{Q}^\top\right]_{jk}+\beta\left[\mathbf{D}\Vb\right]_{jk}} \\
			&=v_{jk}^t\frac{\left[\text{Diag}^2(\bm w)\mathbf{X}^\top\Ub\right]_{jk}+\alpha_s^p\left[\Vb\mathbf{Q}^\top\right]_{jk}+\beta\left[\mathbf{A}\Vb\right]_{jk}}{\left[\text{Diag}^2(\bm w)\Vb\Ub^\top\Ub\right]_{jk}+\alpha_s^p\left[\Vb\mathbf{Q}\mathbf{Q}^\top\right]_{jk}+\beta\left[\mathbf{D}\Vb\right]_{jk}}.
		\end{aligned}
	\end{equation*}
	The result follows Lemma 1 in \citet{lee2001} that $J_{jk}$ is nonincreasing under the iteration $h^{t+1}=\arg\min_h G(h, h^t)$.
	Since the objective function is bounded below by 0, the monotone convergence theorem implies the convergence.

\section*{Appendix C}
\subsection*{Data information}

\begin{enumerate}
	
	\item \textbf{Synthetic dataset}: {This synthetic dataset is generated by a four-component Gaussian mixture model. The data contains six views, with the last two views being noisy.
More specifically, we randomly generate the cluster centers, denoted by $\bm\mu_1,\dots,\bm\mu_4$, for views $\mathbf{X}^{(1)}$ to $\mathbf{X}^{(4)}$. Each element of $\bm\mu_{j}$, $j = 1,...,4$ is independently drawn from the normal distribution with mean randomly generated from a uniform distribution $\mathcal{U}[a,a+10]$ and variance 1.
We set $a=10, 20 ,30 ,40$ for these four views. To generate the covariance matrix for each view, we first generate a random number $b$ from $\mathcal{U}[0.1,1]$, then multiply a symmetric matrix of all ones by $b$. Lastly, we take element-wise power of this matrix by  a symmetric Toeplitz matrix whose diagonals are all 0.
The prior proportions of the 4 components are set to be equal and sum to 1.
Further, we set $\mathbf{X}^{(5)}$ to be the same as $\mathbf{X}^{(1)}$ but with the first 300 observations added by random noises independently generated from $\mathcal{N}(0, 5)$.
We also let $\mathbf{X}^{(6)}$ to be the same as $\mathbf{X}^{(3)}$ but with the first 1000 observations added by random noises independently generated from $\mathcal{N}(0, 10)$.}
	
	\item \textbf{Handwritten digit dataset \footnote{https://archive.ics.uci.edu/ml/datasets/Multiple+Features}}: This dataset contains $2000$ digits and 10 labels.
	Each digit can be decomposed into four views: Fourier coefficients of the character shapes (\textbf{fou}), pixel averages in $2\times 3$ windows (\textbf{pix}), Zernike moments (\textbf{zer}) and profile correlations (\textbf{fac}). 
	

	\item \textbf{Liver hepatocellular carcinoma (LIHC)}: 
	This is a multi-omics dataset used in the application in \cite{seal2020estimating}. 
	Each sample has three different types of measurements (views): gene expression (GE), copy number variation (CNV), and DNA methylation (DNAm). 
	The processed dataset has 404 samples, and the three views have 15397, 16384, 16384  features, respectively.
	To further reduce the dimension, we select the top 100 most highly variable features for each view. In addition, these samples belong to either  tumor or normal samples,  where such class labels are known a priori. 
\end{enumerate}

\subsection*{Empirical analysis on the tuning parameters}
In this section, we show how to select $p$, $\beta$, and $K$ using the synthetic dataset and the handwritten digit dataset. The default values are $p=5$, $\beta=0.01$, and $K=10$. During the experiment, we change the target parameter and fix the remaining ones.

We first analyze the effect of $p$ on the algorithm performance.
Figure \ref{p} shows how the metric scores and the distribution of weights change as $p$ changes. {For the experiments, we let $p\in\{2,4,5,8,11\}$.}
The left panel shows $p$ controls the sparsity of the weight vector, i.e., the effect of different values of $p$ on the distributions of the weight vector.
As $p$ decreases, the weight vector $\bm\alpha$ becomes sparser.
The right panel shows that all the metric scores are close with {$p \in\{4,5,8,11\}$ (a moderate size). }

\begin{figure}[hbt!]
	\centering
	\includegraphics[width=0.7\textwidth]{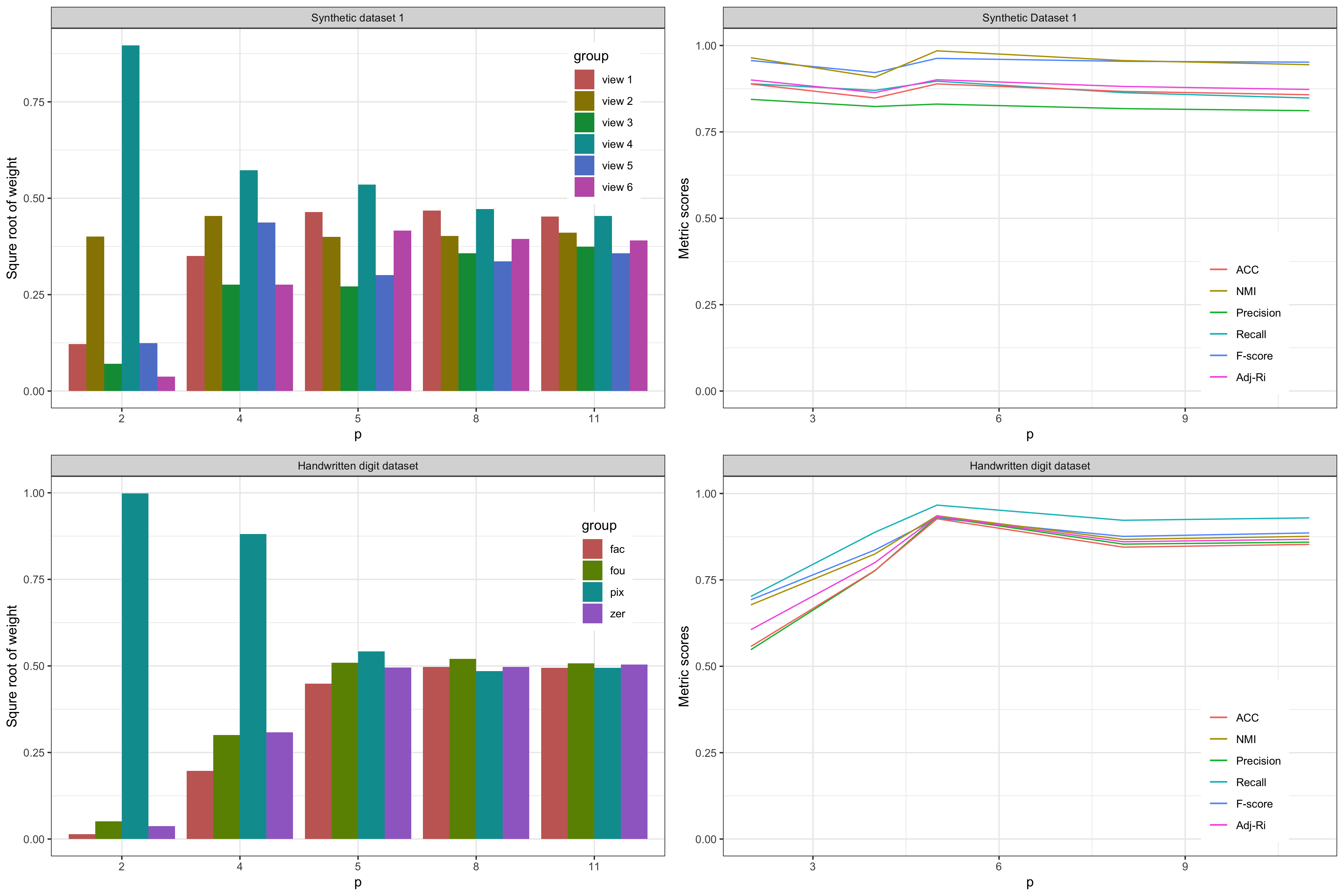}
	\caption{Distributions of view-specific weight $\bm \alpha$ (left) and the metric scores (right) under different values of $p$ for the synthetic dataset and handwritten digit dataset.}
	\label{p}
\end{figure}

Next we empirically illustrate how to choose $\beta$, the manifold regularizer.
Before algorithm implementation, entries in $\mathbf X$ are scaled so that the value $\mathcal{O}_1$ is in general small.
Consequently, we tend to choose a small $\beta$ to balance the matrix factorization effect and the manifold regularization effect. As we can see from Figure \ref{beta}, both datasets demonstrate robust results with different values of $\beta$.
This implies that the clustering performance is  robust to  relatively small $\beta$ values.
Finally, we empirically show how to choose $K$. As we can see from Figure \ref{K}, both datasets demonstrate robust results when $K$ lies in a neighbour of the ground truth.
This means that the choice of $K$ may not affect the clustering results even though it is overestimated or underestimated.

\begin{figure}[hbt!]
	\centering
	\includegraphics[width=0.7\textwidth]{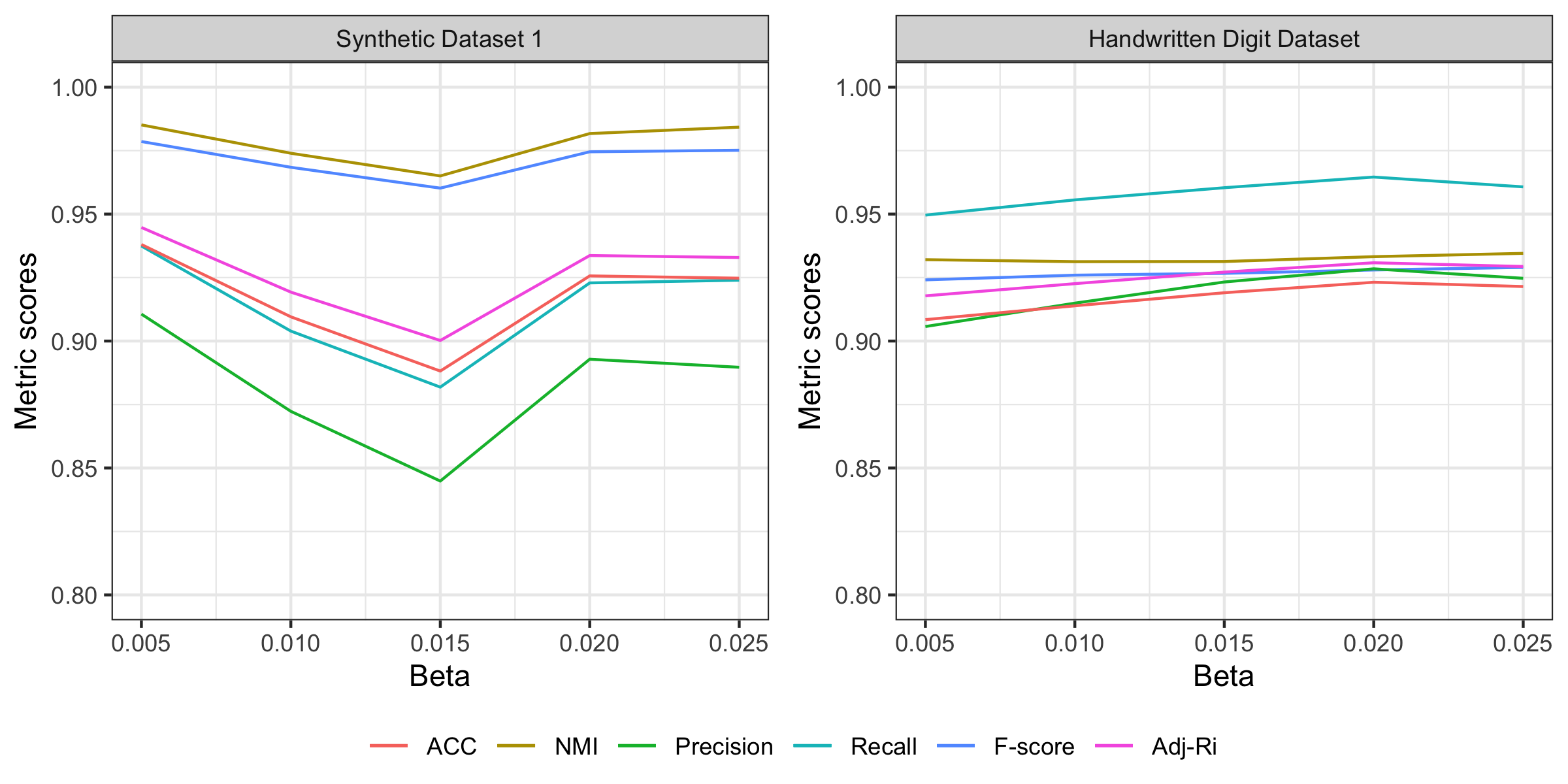}
	\caption{Metric scores under different values of $\beta$ on the synthetic dataset (left) and handwritten digit dataset (right). The x-axis is the value of different $\beta$ from 0.005 to 0.025. Note: we set the limits of the y-axis from 0.8 to 1.}
	\label{beta}
\end{figure}

\begin{figure}[hbt!]
	\centering
	\includegraphics[width=0.7\textwidth]{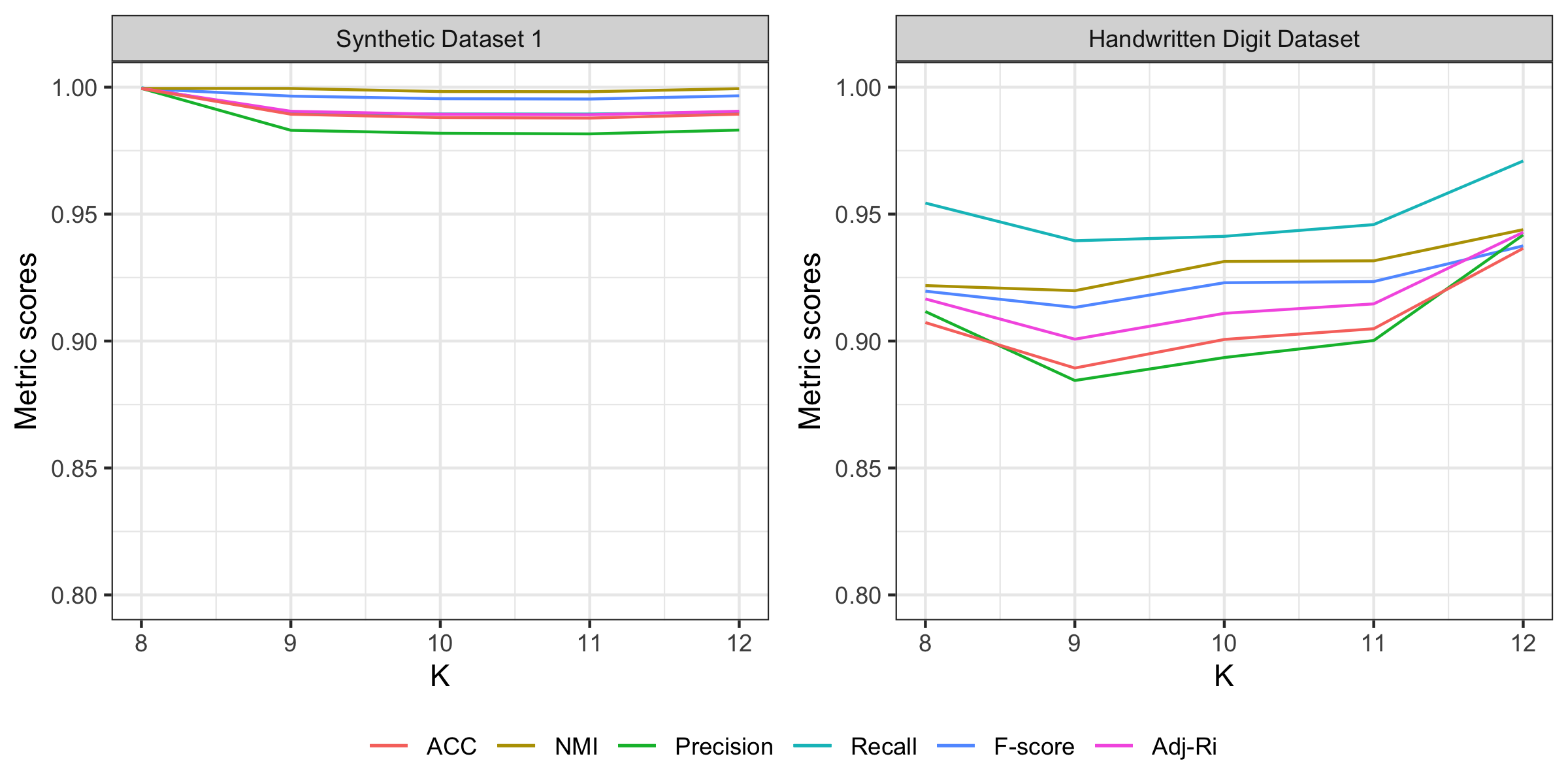}
	\caption{Metric scores under different values of $K$ on the synthetic dataset and handwritten digit dataset. The x-axis shows the value of different $K$ from 8 to 12. Note: we set the limits of the y-axis from 0.8 to 1.}
	\label{K}
\end{figure}

\subsection*{Complexity and convergence study}
\label{sec_comp}
To study the computational complexity, we run a series of experiments on a server with 10 processors and each processor (2.2 GHz Intel Xeon) uses 20GB memory.
We change $N$ and $n_v$ to investigate the corresponding effects.
The default setting is 5000 data points ($N=5000$), 4 views ($n_v=4$) with 10 clusters ($K=10$) and 100 features ($M=100$).
During the experiment, we change one aspect while keeping all others fixed.
The values of $N$ are set to 3k, 5k, 7k, 9k, and 11k, which is larger than $M$, so the theoretical complexity should be quadratic in $N$.
We find the running time is overall linear in terms of smaller $N$ and scales well for large $N$ (e.g. $N>9000$).
Even though the theoretical result for $N>M$ indicates quadratic complexity in $N$, the running time is still acceptable.
For $n_v$, we set its value from 2 to 6.
Row 1 of Figure \ref{con} shows the running time is linear in $n_v$.

\begin{figure*}[hbt!]
	\centering
	\includegraphics[width=0.7\textwidth]{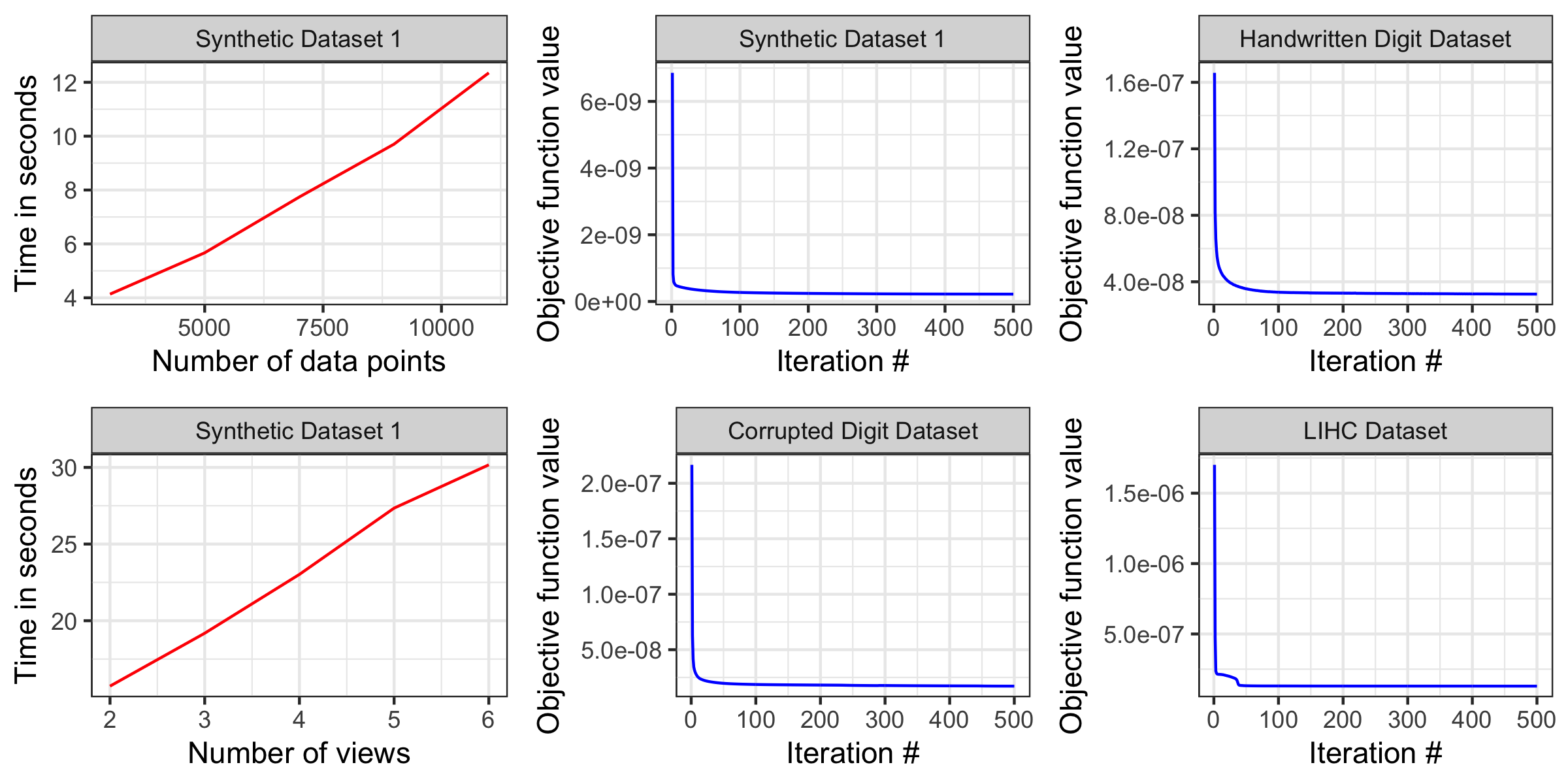}
	\caption{Running time of the WM-NMF algorithm on the synthetic dataset (row 1); convergence curves of WM-NMF algorithm on the  handwritten digit and LIHC dataset (row 2).}
	\label{con}
\end{figure*}

The multiplicative update rule for minimizing the objective function $\mathcal{O}$ is iterative.
Theorem \ref{conv} shows that the algorithm for updating $\mathbf U$ and $\mathbf V$ can converge to a local solution.
Here we investigate how fast the convergence is empirically.
Row 2 of Figure \ref{con} demonstrates the convergence curves.
The x-axis is the number of iterations and y-axis is the objective function value. 
We see that the algorithm converges very fast, usually within 50 iterations.

\section*{Appendix D}
\subsection*{Algorithm complexity analysis}
	\label{cc_analysis}
	We analyze the complexity for updating $\Ub^{(s)}$ and $\Vb^{(s)}$ in the inner iteration. 
	We divide the counts of iterations into multiplication, addition and division.  The overall complexity for the inner iteration is $O(M_sNK+N^2K)$ (we provide the details in appendix A). We summarize the operation counts in Table \ref{cc_sum}.
	
	\begin{table}[hbt!]
	\caption{Computational operation counts for each iteration of $\Ub^{(s)}$ and $\Vb^{(s)}$.}
		\label{cc_sum}
		\centering
		\resizebox{\textwidth}{!}{
			\begin{tabular}{|c|cccc|}
				\hline
				\textbf{}   & \multicolumn{1}{c}{multiplication} & \multicolumn{1}{c}{addition} & \multicolumn{1}{c}{division} & overall \\ \hline
				$\Ub^{(s)}$ &\makecell{$N+M_sN+M_sNK+(N+1)K+$\\$N+KN+2M_sNK+K(N+2)$} &\makecell{$KM_sN+NK+$\\$2M_sNK+K(M_s+N)$} &$M_sK$ &$O(M_sNK)$         \\\hline
				$\Vb^{(s)}$ &\makecell{$N+M_sN+M_sNK+2NK+N^2K$\\$N+KN+2M_sNK+3KN+K$} &\makecell{$KM_sN+N^2K+$\\$2M_sNK+KN^2$} &$NK$ &$O(M_sNK+N^2K)$        \\\hline
		\end{tabular}}
	\end{table}
	
	Further, suppose there are $t_1$ iterations for updating $\mb U^{(s)}$ and $\mb V^{(s)}$ for each view, then the complexity for all views is $O\left\{t_1 n_v (M_*NK+N^2K)\right\}$, where $M_*$ denotes the maximum of $\{M_1,\dots,M_{n_v}\}$.
	After the $t_1$ inner iterations of $\Ub^{(s)}$ and $\Vb^{(s)}$, we still need $O(n_v)$ for  $\alpha_s$, $O(n_v NK)$ for $\Vb^*$, and $O(n_v N)$ for $\text{Diag}(\bm w)$.
	Therefore, for each iteration of the whole procedure of Algoritm~\ref{algorithm1}(lines 4-12), the total complexity is $O\left\{t_1 n_v (M_*NK+N^2K)\right\}$.
	Suppose $t_2$ outer iterations are taken for $\mathcal{O}$ to converge or reaching the maximum number of iteration, then the overall algorithm takes time $O\left\{t_1 t_2 n_v (M_*NK+N^2K)\right\}$ for $N>M_*$.
	When $N<M_*$, we have the overall complexity $O\left\{t_1 t_2 n_v M_*NK\right\}$.
	
	\newpage
	\bibliography{ref}
	\bibliographystyle{plain}
	
\end{document}